\documentclass[10pt,twoside]{IEEEtran}
\usepackage{mathtools, amsthm, amsopn, amsmath}
\usepackage{etex}
\usepackage{graphicx}
\usepackage{endnotes}
\usepackage{hyperref}
\usepackage{epsfig,psfrag}
\usepackage{pst-all}
\usepackage{amssymb,amsfonts,upref,cite,epsf,color,bm}
\usepackage{graphicx}
\usepackage{color}
\usepackage{calc}
\usepackage{booktabs}
\usepackage{tikz}
\usepackage{pgfplots}
\usepackage{csvsimple}
\newcommand\defeq{:=}
\usepackage{subfig}

\newcommand{\SBMparams}{\bm \theta}

\usepackage{algorithm, float}
\usepackage{algpseudocode}
\floatname{algoirthm}{algorithm}
\algnewcommand\algorithmicinput{\textbf{Input:}}
\algnewcommand\INPUT{\item[\Alg.\icinput]}
\algnewcommand\algoirhtmicoutput{\textbf{Output:}}
\algnewcommand\OUTPUT{\item[\algorithmicoutput]}

\newtheorem{proposition}{Proposition}

\newcommand{\nrlabeledcluster}{S}

\DeclareMathOperator*{\argmin}{argmin}

\newcommand\clustersize[1]{N_{#1}}
\newcommand\specgap[1]{\lambda_{#1}}
\newcommand\LapCluster[1]{\mathbf{L}^{(#1)}}
   
\newcommand{\vx}{\mathbf{x}}
 \newcommand{\clusterassgt}{c}
\newcommand{\mD}{\mathbf{D}}
\newcommand{\mL}{\mathbf{L}}

\newcommand{\numnodes}{N}
\newcommand{\numedges}{E}

\newcommand{\gindex}[1][i]{^{(#1)}}

\newcommand{\nrclusters}{K}
\newcommand{\nrnodes}{\numnodes}

\newcommand{\probintra}{p_{\rm in}}
\newcommand{\probinter}{p_{\rm out}}

\newcommand{\partition}{\mathcal{F}}
\newcommand{\cluster}[1]{\mathcal{C}^{(#1)}}
\newcommand{\samplingset}{\mathcal{M}}
\newcommand{\edges}{\mathcal{E}}
\newcommand{\nodes}{\mathcal{V}}
\newcommand{\graph}{\mathcal{G}}

\newtheorem{theorem}{Theorem}
\newtheorem{definition}[theorem]{Definition}

\newcommand{\prob}{{\rm P}}

\definecolor{lavander}{cmyk}{0,0.48,0,0}
\definecolor{violet}{cmyk}{0.79,0.88,0,0}
\definecolor{burntorange}{cmyk}{0,0.52,1,0}

\def\oran{orange!30}

\tikzstyle{vertex}=[draw,circle,burntorange, left color=\oran,
                       text=blue,fill=black!25,minimum size=15pt,inner sep=0pt]
\tikzstyle{sampledvertex}=[draw,circle,black, 
                       text=blue,fill=black!25,minimum size=15pt,inner sep=0pt]                       

\tikzstyle{legendvertex}=[rectangle, right,
                       text=black,minimum size=15pt,inner sep=0pt]

\tikzstyle{legendsample}=[rectangle, right,
                       text=black,minimum size=15pt,inner sep=0pt]  

\title{Clustering in Partially Labeled Stochastic Block Models via Total Variation Minimization}
\author{Alexander Jung\\ 
Department of Computer Science  \\
 Aalto University, Espoo, Finland \\ 
 firstname.lastname(at)aalto.fi
}

\begin{document}
	\maketitle
\begin{abstract}
A main task in data analysis is to organize data points into coherent groups 
or clusters. The stochastic block model is a probabilistic model for the cluster 
structure. This model prescribes different probabilities for the presence of edges within 
a cluster and between different clusters. We assume that the cluster assignments 
are known for at least one data point in each cluster. In such a partially labeled stochastic 
block model, clustering amounts to estimating the cluster assignments of the remaining 
data points. We study total variation minimization as a method for this clustering task. 
We implement the resulting clustering algorithm as a highly scalable message passing 
protocol. We also provide a condition on the model parameters such that total variation 
minimization allows for accurate clustering.  
\end{abstract}


\section{Introduction}

Many application domains generate data with an intrinsic network structure \cite{BigDataNetworksBook,NewmannBook}. 
One of the main workhorses for processing such networked data is the stochastic block model (SBM) \cite{AbbeSBM2018}. 
The SBM is a generative (probabilistic) model for the network structure of data and offers a principled approach 
to community detection or clustering methods \cite{Fortunato2010,Gao2017}. 

The SBM extends the Erd\H{o}s-R{\'e}nyi (ER) random graph model by prescribing an intrinsic cluster structure. 
The cluster assignments of nodes (data points) are considered as labels associated with nodes. 
Clustering algorithms are obtained from inference methods for the SBM which estimate the labels from the 
observed links between data points \cite{Mossel16,Amini2013,PhysRevESSL}. 

Most existing clustering methods for the SBM only take network structure into account. 
However, in some applications we might have a good idea of the (difference in the) 
cluster assignments for a few data points. The partially labeled SBM (PLSBM) assumes 
that cluster assignments of a certain fraction of the nodes are known. Our analysis 
also covers the extreme case of having access to the cluster assignment of exactly 
one data point for each cluster \cite{AvilesRivero2019}.

Clustering methods for the PLSBM, which are also known as local cluster methods, 
have received some attention recently \cite{Mossel16,CaiPLSBM}. Thresholds on the PLSBM 
parameters have been derived in \cite{CaiPLSBM}. These thresholds 
provide conditions which ensure that local clustering methods are successful with 
overwhelming probability for growing graph sizes. In contrast to \cite{Mossel16,CaiPLSBM}, 
our results are non-asymptotic and provide explicit bounds for a given (finite) PLSBM. 

The closest to our setting is \cite{PhysRevESSL}, which studies approximate implementations 
of the Bayes' (optimal) estimator for the cluster assignments. While the results \cite{PhysRevESSL} 
are asymptotic in nature, we use a non-asymptotic analysis using given (finite) model parameters. 
Moreover, the proposed clustering method is not motivated from Bayes' estimation theory but via 
regularized empirical risk minimization using total variation for regularization. We solve this minimization 
problem using a modern convex optimization method (see Section \ref{sec_tv_min}). This iterative method 
allows for rather complete characterization of the convergence behaviour (speed). Computationally, 
the proposed clustering method can be implemented as scalable message passing.

We represent the cluster assignments of data points via piece-wise constant graph signals. 
This opens the toolbox of graph signal processing for the design of clustering methods. 
For the recovery of piece-wise constant graph signals, total variation minimization has been 
proven useful. 

Our main contributions are: 
\begin{itemize}
\item a novel message passing method to solve PLSBM. 
\item a precise condition on the model parameters such that the proposed method 
recovers correct cluster assignments with high probability. 
\end{itemize} 	
	
\textbf{Notation.} 
The maximum and minimum of two numbers $x,y$ is denoted as $x \vee y$ 
and $x \wedge y$, respectively. 
	
\section{Problem Formulation}
\label{sec_problem_formuation} 

We represent networked data by an undirected \emph{empirical graph} $\graph = (\nodes, \edges)$. 
The nodes $i \in \nodes =  \{1, \ldots, \numnodes\}$ represent data points such as text documents or  
social network users.
\begin{figure}[htbp]
\begin{center}
\begin{tikzpicture}
    \tikzset{x=1cm,y=1.2cm,every path/.style={>=latex},node style/.style={circle,draw}} 
    \coordinate[] (x2) at (0,0); 
    \coordinate[] (x10) at (-1,1)  ;    
    \coordinate[] (x3) at (0,2);
    \coordinate[] (x4) at (1,1);  
    \coordinate[] (x12) at (2.5,1) ;   
    \draw[line width=0.4,-] (x2) edge  (x3);
    \draw[line width=0.4,-] (x2) edge (x10);
    \draw[line width=0.4,-] (x3) edge   (x10);
    \draw[line width=0.4,-] (x2) edge    (x4);
    \draw[line width=0.4,-] (x3) edge (x4);
     \coordinate[] (x5) at (5,0); 
    \coordinate[] (x6) at (5,2);    
    \coordinate[] (x7) at (4,1);
    \coordinate[] (x11) at (6,1);          
    \draw[line width=0.4,-] (x5) edge (x6);
        \draw[line width=0.4,-] (x11) edge (x6);
            \draw[line width=0.4,-] (x5) edge (x11);
    \draw[line width=0.4,-] (x5) edge  (x7);
    \draw[line width=0.4,-] (x6) edge  (x7);
    \draw[line width=0.4,-] (x4) edge node [above] {$\partial \partition$}  (x7) ;
    \draw [fill=white]   (x2)  circle (4pt) node [right=1mm] {$2$};
    \draw [fill=white]   (x3)  circle (4pt) node [right=1mm] {$3$};
    \draw [fill=gray!30]   (x10)  circle (4pt) node [above=1mm] {$1$} ;
    \draw [fill=white]   (x4)  circle (4pt) node [above=1mm] {$4$} ;
     \draw [fill=white]   (x5)  circle (4pt) node [right=1mm] {$6$};
    \draw [fill=white]   (x6)  circle (4pt) node [right=1mm] {$7$};
    \draw [fill=gray!30]   (x11)  circle (4pt) node [above=1mm] {$8$} ;
    \draw [fill=white]   (x7)  circle (4pt) node [above=2mm] {$5$} ;
    \node[draw,circle,dashed,minimum size=3cm,inner sep=0pt,label={$\cluster{1}$}] at (0.1,1) {};
    \node[draw,circle,dashed,minimum size=3cm,inner sep=0pt,label={$\cluster{2}$}] at (4.9,1) {};
\end{tikzpicture}
\end{center}
\caption{\label{fig_clustered_graph_signal} Empirical graph $\graph$ whose nodes $\nodes$ are grouped 
into two clusters $\cluster{1}$ and $\cluster{2}$ forming the partition $\partition=\{\cluster{1},\cluster{2}\}$. 
Cluster assignments are known only for shaded nodes.}
\end{figure}

Two data points $i,j \in \nodes$ are connected by an undirected edge $\{i,j\} \in \edges$ if they are considered similar, 
such as documents authored by the same person or social network profiles of befriended users.  
For ease of notation, we denote the edge set $\edges$ by $\{1, \ldots, \numedges\defeq|\edges|\}$. 

It will be convenient to define a directed version of the empirical graph by orienting each undirected 
edge $e=\{i,j\}$ to obtain the directed edge $(e^{+},e^{-})$ with $e^{+} \defeq i \wedge j$ and $e^{-}\defeq i \vee j$. 
For example, the undirected edge $\{ 7, 3\}$ becomes the directed edge $(3,7)$ with $e^{+} = 3$ and $e_{-}=7$. 

The incidence matrix $\mD \!\in\! \mathbb{R}^{\numedges\!\times\!\numnodes}$ of $\graph$ is 
\begin{equation}
D_{e,i} = \begin{cases}  1& \mbox{ if } i = e^{+}  \\ 
- 1& \mbox{ if } i = e^{-}  \\ 
0 &  \mbox{ else.}  \label{equ_def_incidence_mtx}
\end{cases}
\end{equation} 
The rows of $\mD$ correspond to the edges $e\!\in\!\edges$. Each column of $\mD$ 
represents a particular node $i\!\in\!\nodes$ of $\graph$. 
The row representing $e\!=\!\{i,j\}$ contains exactly two non-zero entries 
in the columns corresponding to the nodes $i,j \in \nodes$.

The neighbourhood and degree of a node $i\!\in\!\nodes$ are denoted $\mathcal{N}(i)\!\defeq\!\{j\!:\!(i,j)\!\in\!\edges\}$ and 
$d_{i}\!\defeq\!|\mathcal{N}(i)|$, respectively. We also 
define the directed neighbourhoods of a node $i\!\in\!\nodes$ as 
\begin{align} 
\mathcal{N}^{+}(i) & \defeq \{ j \in \nodes : \{i,j\} \!\in\!\edges, i < j \} \mbox{, and }  \nonumber \\ 
\mathcal{N}^{-}(i) & \defeq \{ j \in \nodes : \{i,j\} \!\in\!\edges, i > j \}. 
\end{align} 

Another matrix that is naturally associated with a graph $\graph$ is the Laplacian matrix 
$\mL \in \mathbb{R}^{\nrnodes \times \nrnodes}$ with entries 
\begin{equation} 
L_{i,j}\!=\!-1 \mbox{ for }  \{i,j\}\!\in\!\edges, L_{i,i}\!=\!d_{i} \mbox{ for } i\!\in\!\nodes, L_{i,j}\!=\!0 \mbox{ else.}
\end{equation} 

The Laplacian matrix $\mathbf{L}$ is positive semi-definite with non-negative ordered 
eigenvalues $\lambda_{1} \leq \lambda_{2} \leq \ldots \leq \lambda_{\nrnodes}$. 
A graph is connected if and only if $\lambda_{1}=0$  \cite{ChungSpecGraphTheory}. 
Our approach focuses on the second largest eigenvalue $\lambda_{2}$ which quantifies 
the connectivity of $\graph$. Loosely speaking, the larger the value of $\lambda_{2}$, the 
better the graph is connected. 
	
We consider networked data with an intrinsic cluster (community) 
structure \cite{NewmannBook,Fortunato2010,Porter2009}. Formally, 
the data points are partitioned into $\nrclusters$ disjoint clusters (or communities), 
\begin{equation}
\label{equ_partitioning}
\nodes = \cluster{1} \cup \ldots \cup \cluster{\nrclusters}.  
\end{equation}  
The size of the $k$th cluster is $\clustersize{k} \defeq \big| \cluster{k}\big|$, 
$\nrnodes = \sum_{k=1}^{\nrclusters} \clustersize{k}$. 

Overloading notation, $\cluster{k}$ denotes also the subgraph 
induced by cluster $\cluster{k}$ containing edges $\{i,j\}\!\in\!\edges$ with $i,j\!\in\!\cluster{k}$. 
The Laplacian matrix of $\cluster{k}$ is denoted $\LapCluster{k}$. 

Each data point $i\!\in\!\nodes$ belongs to exactly one cluster, whose index 
is denoted $\clusterassgt\gindex \in \{1,\ldots,\nrclusters\}$. The cluster assignments 
are known for a small subset $\samplingset$ of nodes only (e.g., via domain expertise). 
Our goal is to recover the cluster assignments of all data points 
based on the network structure and known cluster assignments in $\samplingset$. 

The recovery of cluster assignments $\clusterassgt\gindex$ is based on the hypothesis 
that data points within cluster are more densely connected than data points of 
different clusters. A popular model for networks with a cluster or community structure is the 
stochastic block model (SBM).

As a probabilistic model, the SBM interprets the (presence of) edges between 
two nodes $i,j \in \nodes$ as realizations of independent binary random variables $t_{i,j} \in \{0,1\}$. 
An edge connects two nodes $i,j\!\in\!\nodes$ if and only if $t_{i,j}\!=\!1$. 
The edge probability $\prob\{t_{i,j}\!=\!1\}$ between nodes $i,j\!\in\!\nodes$ depends 
only on the cluster assignments $c_{i}, c_{j}\!\in\!\{1,\ldots,\nrclusters\}$. 

The most basic SBM prescribes a constant probability $\probintra$ 
for edges within the same cluster, and another constant probability $\probinter$ for intra-cluster 
edges,   
\begin{equation}
\label{equ_def_prob_inter}
\prob \{ t_{i,j} \!=\! 1\}\!=\! \probintra \mbox{ for } \clusterassgt^{(i)} = \clusterassgt^{(j)} \mbox{, } 
\prob \{ t_{i,j} \!=\! 1\}\!=\! \probinter  \mbox{ else.}
\end{equation}
We denote the complete set of SBM parameters as 
\begin{equation}
\label{equ_def_SBM_parameters}
\SBMparams \defeq \big( \clustersize{1},\ldots,\clustersize{\nrclusters},\probintra,\probinter \big).  
\end{equation} 
	
Our main result is the characterization of SBM parameter regime such that 
the cluster assignments $\clusterassgt_i$ of all data points $i \in \nodes$ 
can be perfectly recovered. We characterize a value range for the SBM parameters 
$\SBMparams$ such that a convex optimization method recovers 
all cluster assignments $\clusterassgt_{i}$. 

The accuracy of clustering methods depends on the connectivity structure 
induced by the edges $\edges$ as well on the knowledge of the cluster assignments $\clusterassgt_i$ 
for few data points in a small (training or sampling) set $\samplingset \subseteq \nodes$. 
We assume that cluster assignments are known for $\nrlabeledcluster$ nodes within 
each cluster, 
\begin{equation}
\label{equ_sampling_set_one_entry_each_cluster}
\big| \samplingset \cap \cluster{k}  \big| = \nrlabeledcluster \mbox{ for each } k=1,\ldots,\nrclusters. 
\end{equation} 
The set $\samplingset$ contains $\nrlabeledcluster$ nodes $i_{1}^{(k)},\ldots,i_{\nrlabeledcluster}^{(k)}$ 
for each $\cluster{k}$. 

\section{TV Minimization} 
\label{sec_tv_min}

We frame the task of learning the clustering assignments in a SBM within 
graph signal processing. To this end, we encode the cluster assignments 
$\clusterassgt_{i}$ as $\nrclusters$ separate graph signals 
\begin{equation}
\label{equ_def_cluster_signal}
x_{i}^{(k)}\!=\!1 \mbox{ for } i \in \cluster{k},  x_{i}^{(k)} =0 \mbox{ otherwise.} 
\end{equation} 
The piece-wise constant graph signals $x_{i}^{(k)}$, for $k=1,\ldots,\nrclusters$ are the indicator 
signals (functions) for the clusters $\cluster{k}$. 
Trivially, finding the cluster assignments $\clusterassgt_{i}$ is equivalent to the problem 
of recovering the graph signals $x_{i}^{(k)}$ from its values on the 
set $\samplingset$ of data points with known cluster assignments. The 
recovery of the $x_{i}^{(k)}$ can be down separately. In what follows, 
we will consider an arbitrary but fixed cluster $\cluster{k}$ and denote 
the corresponding indicator signal by $x_{i} \defeq x_{i}^{(k)}$.


Our focus is on the SBM parameter $\SBMparams$ regime where the nodes 
within the same cluster $\cluster{k}$ are connected densely, whereas only 
few links exist between different clusters.In these regimes, the graph signal 
$x_{i}$ will typically have a small TV
\begin{align}
\label{equ_def_TV_norm}
\| \vx \|_{\rm TV} & \defeq \sum_{\{i,j\}\in \edges} | x_{j} -x_{i} |. 
\end{align}
The TV of a graph signal on a subset $\mathcal{S} \subseteq \edges$ of edges is 
\begin{align}
\label{equ_def_TV_norm_subs}
\|  \vx \|_{\mathcal{S}} & \defeq \sum_{\{i,j\}\in \mathcal{S}}  | x_{j} -x_{i} |. 
\end{align}

It seems natural to recover a graph signal with small TV and 
known signal values on the set $\samplingset$ via TV minimization 
\begin{align} 
\widehat{\vx} &\!\in\!\argmin_{\tilde{\vx} \in \mathbb{R}^{\nrnodes}} \underbrace{\sum_{\{i,j\} \in \edges} \hspace*{-3mm}   | \tilde{x}_{j}\!-\!\tilde{x}_{i}|}_{=  \| \tilde{\vx} \|_{\rm TV}}
 \mbox{s.t. }  \tilde{x}_{i}\!=\!x_{i}  \mbox{ for all } i\!\in\!\samplingset. \label{equ_min_constr}
\end{align}
Since the objective function and the constraints in \eqref{equ_min_constr} are convex, the optimization 
problem \eqref{equ_min_constr} is a convex optimization problem \cite{BoydConvexBook}. In fact, 
\eqref{equ_min_constr} can be reformulated as a linear program \cite[Sec. 1.2.2]{BoydConvexBook}. 

As the notation in \eqref{equ_min_constr} indicates, there might be more than one solution. 
Each solution $\hat{\vx}$ of \eqref{equ_min_constr} provides an estimate for the cluster 
indicator $x^{(k)}_{i}$ (see \eqref{equ_def_cluster_signal}) that is characterized by (i) it is 
consistent with the cluster indicator for all nodes $i\in \samplingset$; and 
(ii) it has minimum TV \eqref{equ_def_TV_norm} among all such graph signals. 

As shown in \cite{JungTVMin2019}, TV minimization \eqref{equ_min_constr} can be solved iteratively 
by a scalable message passing method which we have summarized in Algorithm \ref{sparse_label_propagation_mp}. 
The stopping criterion in Algorithm \ref{sparse_label_propagation_mp} can be a fixed number of iterations, 
chosen based on the convergence analysis in \cite{ComplexitySLP2018}. 

Algorithm \ref{sparse_label_propagation_mp} solves a separate TV 
minimization problem \eqref{equ_min_constr} for each cluster $k=1,\ldots,K$. 
For each $k=1,\ldots,K$, an estimate $\bar{x}^{(k)}_{i}$ for the cluster indicator 
signal $x^{(k)}_{i}$ (see \eqref{equ_def_cluster_signal}) is computed. The final 
estimate $\hat{c}_{i}$ for the cluster assignment of node $i\in\!\nodes$ is chosen 
such that  
\begin{equation} 
\bar{x}^{(\hat{c}_{i})}_{i} = \max_{k=1,\ldots,\nrclusters} \bar{x}^{(k)}_{i}. 
\end{equation}

\begin{algorithm}[h]
\caption{Clustering in PLSBM via TV Minimization}{}
\begin{algorithmic}[1]
\renewcommand{\algorithmicrequire}{\textbf{Input:}}
\renewcommand{\algorithmicensure}{\textbf{Output:}}
\Require empirical graph $\graph=(\nodes,\edges)$, node set $\samplingset$ 
with known cluster assignments $\{ c_i \}_{i \in \samplingset}$. 
\vspace*{1mm}
 \For{ each cluster $k=1,\ldots,K$ } 
\vspace*{3mm}
\State $r\!\defeq\!0$, $\bar{\vx}^{(k)}\!=\!\hat{\vx}^{(0)}\!=\!\hat{\vx}^{(-1)}\!=\!\hat{\vx}^{(0)}\!\defeq\!\mathbf{0}$, 
$\gamma_{i}\!\defeq\!1/d_{i}$.
\Repeat
\vspace*{0.5mm}
\State $\tilde{x}_i\!\defeq\!2 \hat{x}^{(r)}_{i}\!-\!\hat{x}^{(r-1)}_{i}$ for all $i\!\in\!\nodes$ 
\vspace*{0.5mm}
\State  \hspace*{-2mm}$\hat{y}_{e}^{(r\!+\!1)}\!\defeq\!\hat{y}^{(r)}_{e}\!+\!(1/2)  (\tilde{x}_{e^{+}}\!-\!\tilde{x}_{e^{-}})$ for all $e\!\in\!\edges$
\vspace*{0.5mm}
\State \hspace*{-2mm}$\hat{y}_{e}^{(r\!+\!1)}\!\defeq\!\hat{y}_{e}^{(r\!+\!1)} / \max\{1, |\hat{y}_{e}^{(r\!+\!1)}| \}$ for all $e\!\in\!\edges$ 
\vspace*{0.5mm}
\State for all nodes $i\!\in\!\nodes$: 
\begin{equation}  
\nonumber
\hspace*{-3mm}\hat{x}^{(r\!+\!1)}_{i}\!\defeq\!\hat{c}^{(r)}_{i}\!-\!\gamma_{i} \bigg[ \hspace*{0mm}\sum\limits_{j \in \mathcal{N}^{+}(i)} \hspace*{-1mm}\hat{y}^{(r\!+\!1)}_{(i,j)} \hspace*{-1mm}- \hspace*{-1mm}\sum\limits_{j \in \mathcal{N}^{-}(i)} \hspace*{-1mm}\hat{y}^{(r\!+\!1)}_{(j,i)} \hspace*{0mm}\bigg]   
\end{equation}
\State $\hat{x}^{(r\!+\!1)}_{i} \defeq x_{i}^{(k)}$ for all $i\!\in\!\samplingset$ 
\vspace*{0.5mm}
\State $r \defeq r\!+\!1$    
\vspace*{0.5mm}
\State \hspace{-3mm}$\bar{x}^{(k)}_{i} \defeq (1-1/r)\bar{x}^{(k)}_{i} + (1/r) \hat{x}^{(r)}_{i}$ for all $i\!\in\!\nodes$
\vspace*{2mm}
\Until{stopping criterion is satisfied}
\EndFor
\vspace*{1mm}
\Ensure $\hat{c}_{i}\!\defeq\! {\rm argmax}_{k\!=\!1,\ldots,\nrclusters} \bar{x}^{(k)}_{i}$ for all $i\!\in\!\nodes$
\end{algorithmic}
\label{sparse_label_propagation_mp}
\vspace*{-1mm}
\end{algorithm}

\section{When Does it Work?}
\label{sec_num}
We first state our main result and then devote the rest of the section to its derivation. 
\begin{theorem} 
Consider the empirical graph $\graph$ modelled as the realization of a SBM 
with parameters ${\bm \theta}$ \eqref{equ_def_SBM_parameters}. For each 
cluster $\cluster{k}$, we know (the cluster assignments of) 
$\nrlabeledcluster$ data points $i_{1}^{(k)},\ldots,i_{\nrlabeledcluster} \in \cluster{k}$ \eqref{equ_sampling_set_one_entry_each_cluster}). 
Then, if 
	\begin{equation} 
	\label{equ_cond_param_recovery}
	\nrlabeledcluster \probintra/\probinter  \geq  \beta  \clustersize{k}  (\numnodes-\clustersize{k})
	\end{equation}
the probability of Algorithm \ref{sparse_label_propagation_mp} failing to deliver the correct cluster assignments is 
upper bounded by 
\begin{equation} 
\sum_{k=1}^{\nrclusters} \exp\big(-\probinter \clustersize{k}(\numnodes\!-\!\clustersize{k})  \alpha\big)+
 (\clustersize{k}\!-\!1) 0.9^{\probintra \clustersize{k}/2}.
\end{equation} 
Here, $\alpha$ and $\beta$ are (small) numerical constants. 
\label{thm_main_result}	
\end{theorem} 
We derive conditions on the SBM parameters ${\bm \theta}$ (see \eqref{equ_def_SBM_parameters})
such that solutions of TV minimization \eqref{equ_min_constr} coincide with 
the cluster indicator signals \eqref{equ_def_cluster_signal} with high probability. This implies, in turn, 
that Algorithm \ref{sparse_label_propagation_mp} delivers the correct cluster assignments for all nodes, 
$\hat{c}_{i} = c_{i}$ for all $i \in \nodes$. 
To derive these conditions, we use the concept of circulations \cite{KleinbergTardos2006,JungnckelBook}. 

A circulation $f$ assigns each directed edge $e\!=\!(i,j)\!\in\!\edges$ some 
value $f_{e}\!\in\!\mathbb{R}$. 
We interpret the value $f_{e}$ as the amount of flow that passes through 
the edge $e$ from $e^{(+)}$ to $e^{(-)}$. 
The net flow into any node must be zero, 
\begin{equation} 
\label{equ_conservation_law}
\hspace*{-3mm}\sum_{j \in \mathcal{N}^{+}(i)} \hspace*{-2mm}f_{(i,j)}\!-\!\sum_{j \in \mathcal{N}^{-}(i)} \hspace*{-2mm} f_{(j,i)} = 0 \mbox{ for each } i\!\in\!\nodes.
\end{equation} 

Circulations provide a device to quantify the connectivity of the clusters $\cluster{k}$. 
We will develop a method that, for a particular cluster $\cluster{k}$, allows to probe 
how well the labelled node $i^{(k)} \in \samplingset$ is connected with the cluster 
boundary 
\begin{equation}
\partial \cluster{k} \defeq \{ i \in \cluster{k}: j \notin \cluster{k} \mbox{ with } \{i,j\} \in \edges \}. 
\end{equation}

It will be convenient to associate each 
cluster $\cluster{k}$ with a modified subgraph of the empirical graph $\graph^{(k)}$. 
\begin{definition}
Consider some $\cluster{k}$ containing one labeled node $i^{(k)}$ with 
known cluster assignment $\clusterassgt_{i}=k$. We construct the augmented subgraph 
$\widetilde{\graph}^{(k)}$ associated with $\cluster{k}$ by 
by retaining all intra-cluster edges $\{i,j\} \in \edges$ with $i,j \in \cluster{k}$. 
Moreover, the augmented subgraph $\widetilde{\graph}^{(k)}$ contains an additional 
node $t^{(k)}$ along with an edge $(t^{(k)},i)$ to each boundary node $i \in \partial \cluster{k}$.
\end{definition} 
Figure \ref{fig:augmented_subgraph} illustrates the concept of augmented subgraphs.
\vspace*{-2mm}
\begin{figure}[h]
	\vspace*{-2mm}
	\includegraphics[width=\columnwidth]{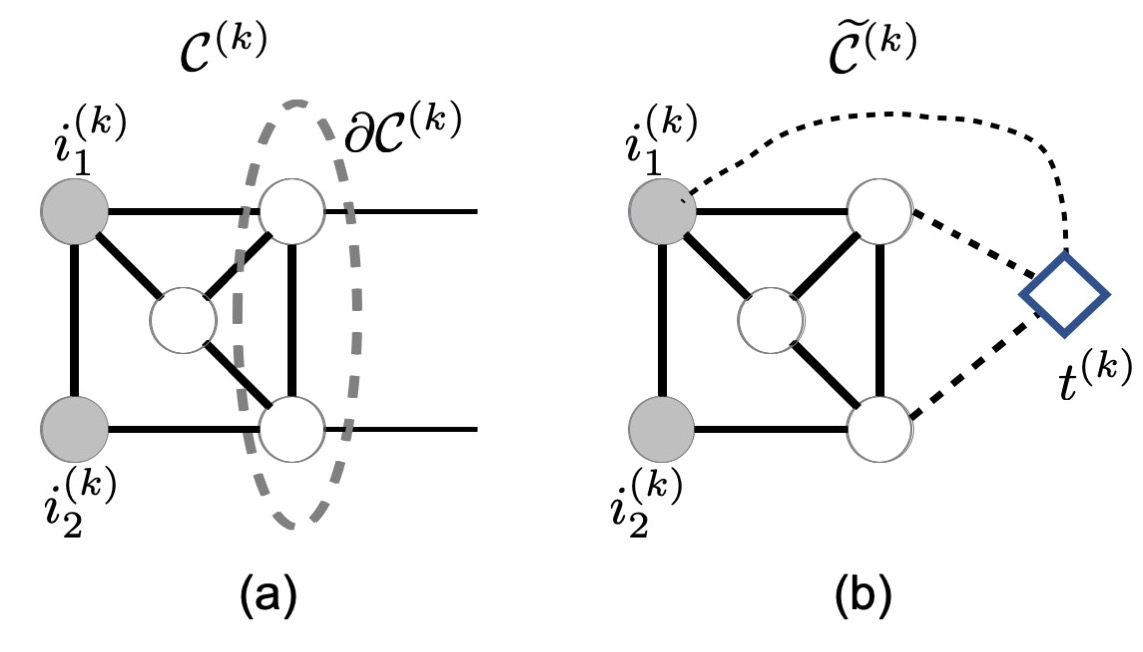}
	\vspace*{-8mm}
	\caption{(a) Cluster $\cluster{k}$ with boundary nodes $\partial \cluster{k}$. 
		Nodes with known cluster assignment are shaded. ( b) Augmented sub-graph $\widetilde{\graph}$ 
		obtained from cluster $\cluster{k}$ by adding node $t^{(k)}$ and edges $(i,t^{(k)})$ for each 
		boundary node $i\in \partial \cluster{k}$.}
	\label{fig:augmented_subgraph}
	\vspace*{0mm}
\end{figure} 
Definition \ref{fig:augmented_subgraph} is conceptually similar to the 
definition \cite[Definition 6]{JungTVMin2019}. In contrast to Definition \ref{fig:augmented_subgraph}, 
which is tailored to unweighted graphs,  \cite[Definition 6]{JungTVMin2019} applies to 
weighted graphs (with edges being assigned weights). 

Roughly speaking, the nodes $i_{1}^{(k)},\ldots,i_{\nrlabeledcluster}^{(k)}$ are well connected to the boundary 
$\partial \cluster{k}$ if there exist circulations with large flows through the boundary edges 
incident to $\partial \cluster{k}$. To make this approach sensible, we need to restrict the 
flows through intra cluster edges of $\cluster{k}$ via capacity constraints 
\begin{equation} 
\label{equ_capacity_constraints}
\hspace*{-3mm}|f_{e}|\!\leq\!1  \mbox{ for any edge } e\!=\!\{i,j\}\!\in\!\edges \mbox{ such that } i,j\!\in\!\cluster{k}. 
\end{equation} 
We formalize this approach in order to quantify the connectivity 
between labeled nodes $\samplingset$ and cluster boundaries. 
\begin{definition}
\label{def_well_connected}
Consider some cluster $\cluster{k}$ which contains one node $i^{(k)}$ for which we know 
the cluster assignment $\clusterassgt$. The node $i^{(k)}$ is well connected to the cluster 
boundary $\partial \cluster{k}$, if for each choice of boundary weights 
$\big\{ w_{e} \in \{-2,2\} \big\}_{e \in \partial \cluster{k}}$, there exists is a 
circulation $f_{e}$ on $\widetilde{\graph}^{(k)}$ such that 
\begin{equation} 
f_{e} = w_{e} \mbox{ for all edges } e=\big( t^{(k)},i \big) \mbox{ with } i \in \partial \cluster{k}.
\end{equation}
The circulation $f_{e}$ has to satisfy the capacity constraints \eqref{equ_capacity_constraints} 
on intra-cluster edges. 
\end{definition}
Carefully note that Definition \eqref{def_well_connected} does not constrain the 
flow $|f_{e}|$ through the particular edge $e = \big( i^{(k)},t^{(k)} \big)$. 

If the nodes $i^{(k)}\!\in\!\samplingset$, with known cluster assignments, 
are well connected (in the sense of Definition \ref{def_well_connected}) 
to the cluster boundaries, TV minimization \eqref{equ_min_constr} succeeds 
in recovering the cluster indicator signals \eqref{equ_def_cluster_signal} for every node. 
\begin{proposition}[Thm.\ 3 in \cite{NNSPFrontiers2018}]
\label{prop_well_connected_clusters}
Consider an empirical graph $\graph$ partitioned into disjoint clusters 
$\cluster{1},\ldots,\cluster{\nrclusters}$ (see \eqref{equ_partitioning}). The cluster 
assignments are known for one node $i^{(k)}$ in each cluster $\cluster{k}$. If, for each 
cluster, the labelled node is well connected to the cluster boundary, TV minimization \eqref{equ_min_constr} 
delivers the correct cluster assignments for all nodes $i \in \nodes$. 
\end{proposition}

To use Proposition \ref{prop_well_connected_clusters}, we need to ensure that 
the nodes $i^{(k)} \in \samplingset$ are well connected to the cluster boundaries $\partial \cluster{k}$. 
To this end, we first relate the requirements of Definition \ref{def_well_connected} to 
the size of cuts in the augmented subgraph $\widetilde{\graph}^{(k)}$. 
\begin{proposition} 
\label{prop_suff_cond_cutsize}
Consider some cluster $\mathcal{C}^{(k)}$ of the empirical graph $\graph$ 
containing the labeled node $i^{(k)}$. Assume that for each subset $\mathcal{S} \subseteq \cluster{k}$, 
containing $s \defeq \big| \mathcal{S} \cap \partial \cluster{k}\big|$ boundary nodes, there are at least 
$2 s$ edges between $\mathcal{S}$ and $\mathcal{C}^{(k)} \setminus \mathcal{S}$. Then $i^{(k)}$ is well 
connected to the boundary $\partial \cluster{k}$. 
\end{proposition} 
\begin{proof}
Application of Hoffman's circulation theorem \cite[Thm. 10.2.7]{JungnckelBook}. 
See also \cite{Hoffman1960} and \cite[Ch. 3]{BertsekasNetworkOpt}. 
\end{proof}
Let  $E^{(\mathcal{S})}$ denote the number of edges connecting some node 
in $\mathcal{S}\subset \cluster{k}$ with some node in $\cluster{k} \setminus \mathcal{S}$. 
The condition in Proposition \ref{prop_suff_cond_cutsize} can be ensured by 
requiring 
\begin{equation} 
\label{equ_suff_cond_flow}
E^{(\mathcal{S})} \geq 2 |\partial \cluster{k}| \mbox{ for each subset } \mathcal{S} \subset \cluster{k} \setminus \{ i^{(k)} \}.
\end{equation} 

In order to identify SBM parameter regimes where \eqref{equ_suff_cond_flow} is 
satisfied with high probability we use spectral graph theory \cite{ChungSpecGraphTheory}. 
In particular, we relate the number $E^{(\mathcal{S})}$ of edges leaving a subset 
$\mathcal{S} \subset \cluster{k}$ to the smallest non-zero eigenvalue of the 
Laplacian matrix $\LapCluster{k}$ \cite[Ch. 2]{ChungSpecGraphTheory}
\begin{equation} 
\label{equ_cut_specgaph}
E^{(\mathcal{S})} \geq  (1-1/\numnodes) \lambda_{2}\big( \LapCluster{k} \big). 
\end{equation} 
Combining \eqref{equ_cut_specgaph} with \eqref{equ_suff_cond_flow}, a sufficient condition 
for the requirements in Proposition \ref{prop_suff_cond_cutsize} is 
\begin{equation} 
\label{equ_suf_condcut_specgaph}
(1-1/\numnodes) \lambda_{2}\big(\LapCluster{k} \big) \geq 2 |\partial \cluster{k}|
\end{equation} 
for each cluster $\cluster{k}$ of the empirical graph $\graph$. 

We now turn to characterizing regimes of the SBM parameters 
$\clustersize{1},\ldots,\clustersize{\nrclusters},\probintra,\probinter$ such 
that the condition \eqref{equ_suf_condcut_specgaph} is satisfied. To this end, 
we will use concentration results to control the probability of the 
left or right hand side in \eqref{equ_suf_condcut_specgaph} to deviate too 
much from their corresponding expected values. 

Let us begin with analyzing the boundary size $|\partial \cluster{k}|$. We can represent 
this quantity as the sum of $\clustersize{k}(\numnodes-\clustersize{k})$ i.i.d.\ Bernoulli 
variables $t_{i,j}$ (see \eqref{equ_def_prob_inter}),
\begin{equation} 
|\partial \cluster{k}| = \sum_{i \in \cluster{k}, j \notin \cluster{k}} t_{i,j}.
\end{equation}  
Using well-known concentration inequalities (see, e.g., \cite[Cor. 7.31]{RauhutFoucartCS}, 
\begin{align} 
\label{equ_large_dev_cluster_boundary}
\prob \big\{|\partial \cluster{k}| \geq 2 \probinter \clustersize{k}(\numnodes-\clustersize{k}) \big\} & \leq \nonumber \\
& \hspace*{-15mm} \exp\big(- \probinter \clustersize{k}(\numnodes-\clustersize{k})  \alpha\big).
\end{align} 

We now turn to analyzing the large deviations of the LHS in \eqref{equ_suf_condcut_specgaph}. 
The second smallest eigenvalue $\specgap{k} \defeq \lambda_{2}\big(\LapCluster{k}\big)$ 
of the Laplacian of the Erd{\"o}s-R{\'e}nyi graph $\cluster{k}$ satisfies \cite{Tropp2015} 
\begin{equation}
\label{equ_large_dev_spec_gap}
\prob \big\{ \specgap{k} \leq (1/2) \probintra  \clustersize{k} \big\} \leq (\clustersize{k}- 1) 0.9^{\probintra \clustersize{k}/2}.
\end{equation} 

The statement of Theorem \ref{thm_main_result} follows then 
from  combining \eqref{equ_large_dev_cluster_boundary} 
and \eqref{equ_large_dev_spec_gap}, with a union bound.  

\section{Numerical Experiments}
\label{sec_SBM_graph}

We verify Theorem \ref{thm_main_result} numerically using a PLSBM with 
clusters $\cluster{1}\!=\!\{1,\ldots,50 \}$ and $\cluster{2}\!=\!\{51,\ldots,100\}$. 

An edge is placed between nodes $i,j$ with probability $\probintra$ if they are in 
the same cluster and with probability $\probinter$ if they are from different clusters. 
using Algorithm \ref{sparse_label_propagation_mp} and, we recover the cluster assignments $c_{i}$ for all nodes 
based on knowing them for $\nrlabeledcluster$ nodes in each cluster $\cluster{k}$. 

%
%

Fig.\ \ref{fig_NMSEconnect_SBM} depicts the accuracy of Alg.\ \ref{sparse_label_propagation_mp} 
for varying $\probintra/\probinter$ and number $\nrlabeledcluster$ of labeled data points (per cluster). 
The accuracy is the fraction of nodes $i\!\notin\!\samplingset$ for which $\hat{c}_{i} = c_{i}$ 
(averaged over $10$ i.i.d.\ simulation runs). Agreeing with \eqref{equ_cond_param_recovery}, the 
accuracy of Alg.\ \eqref{sparse_label_propagation_mp} depends on the SBM parameters 
${\bf \theta}$ via the quantity $\nrlabeledcluster\probintra/\probinter$, for fixed cluster size $\clustersize{k}$. 

\begin{figure}[htbp]
\begin{center}
\begin{tikzpicture}
 \tikzset{x=0.03cm,y=5cm,every path/.style={>=latex},node style/.style={circle,draw}}
    \csvreader[ head to column names,%
                late after head=\xdef\aold{\a}\xdef\bold{\b},,%
                after line=\xdef\aold{\a}\xdef\bold{\b}]%
                {ACCoverSBMParam05.csv}{}
                {\draw [line width=0.0mm] (\aold, \bold) (\a,\b) node {\large $\circ$};
               }
           
              \csvreader[ head to column names,%
          late after head=\xdef\aold{\a}\xdef\bold{\b},,%
          after line=\xdef\aold{\a}\xdef\bold{\b}]%
          {ACCoverSBMParam10.csv}{}
          {\draw [line width=0.0mm] (\aold, \bold) (\a,\b) node {\large $\times$};
          }
      
                    \csvreader[ head to column names,%
      late after head=\xdef\aold{\a}\xdef\bold{\b},,%
      after line=\xdef\aold{\a}\xdef\bold{\b}]%
      {ACCoverSBMParam15.csv}{}
      {\draw [line width=0.0mm] (\aold, \bold) (\a,\b) node {\large $\star$};
      }
          \draw[->] (0,0.5) -- (230,0.5); 
      \draw[->] (0,0.49) -- (0,1.1);
      
        \node [anchor=west] at (100,0.8) {$\nrlabeledcluster\!=\!5$ ('$\circ$')};
         \node [anchor=west] at (100,0.7) {$\nrlabeledcluster\!=\!10$ ('$\times$')};
         \node [anchor=west] at (100,0.6) {$\nrlabeledcluster\!=\!15$ ('$\star$')};
        
      \node [anchor=west] at (0.2,1.1) {accuracy};
            \foreach \label/\labelval in {0.5/$0.5$,1/$1$}
        { 
          \draw (1pt,\label) -- (-1pt,\label) node[left] {\large \labelval};
        }
        
              \foreach \label/\labelval in {0/$0$,40/$40$,80/$80$,120/$120$,160/$160$,200/$200$}
        { 
          \draw (\label,0.51) -- (\label,0.49) node[below] {\large \labelval};
        }
         \node [anchor=west] at (180.1,0.6) {\centering $\nrlabeledcluster \probintra/\probinter$};
        \vspace*{-3mm}
\end{tikzpicture}
        \vspace*{-3mm}
\end{center}
  \caption{Clustering accuracy achieved by Algorithm\ \ref{sparse_label_propagation_mp} for varying $\probintra/\probinter$. 
   }
  \label{fig_NMSEconnect_SBM}
  \vspace*{-3mm}
\end{figure}
The source code underlying this experiment is available at \url{https://github.com/alexjungaalto/ResearchPublic/blob/master/TVMinPLSBM/tvmin_PLSBM.m}. 

\section{Acknowledgement}
The author wishes to thank Lasse Leskel{\"a} for helpful discussions.


\bibliographystyle{IEEEbib}
\bibliography{/Users/alexanderjung/Literature}

\begin{thebibliography}{10}

\bibitem{BigDataNetworksBook}
S.~Cui, A.~Hero, Z.-Q. Luo, and J.M.F. Moura, Eds.,
\newblock {\em Big Data over Networks},
\newblock Cambridge Univ. Press, 2016.

\bibitem{NewmannBook}
M.~E.~J. Newman,
\newblock {\em Networks: An Introduction},
\newblock Oxford Univ. Press, 2010.

\bibitem{AbbeSBM2018}
E.~Abbe,
\newblock ``Community detection and stochastic block models: Recent
  developments,''
\newblock {\em Journal of Machine Learning Research}, vol. 18, no. 177, pp.
  1--86, 2018.

\bibitem{Fortunato2010}
S.~Fortunato,
\newblock ``Community detection in graphs,''
\newblock {\em Physics Reports}, vol. 486, no. 3-5, pp. 75--174, Feb. 2010.

\bibitem{Gao2017}
C.~Gao, Z.~Ma, A.~Y. Zhang, and H.~H. Zhou,
\newblock ``Achieving optimal misclassification proportion in stochastic block
  models,''
\newblock {\em J. Mach. Learn. Res.}, vol. 18, pp. 1--45, 2017.

\bibitem{Mossel16}
E.~Mossel, J.~Neeman, and A.~Sly,
\newblock ``Belief propagation, robust reconstruction and optimal recovery of
  block models,''
\newblock {\em Ann. App. Prob.}, vol. 26, no. 4, pp. 2211--2256, 2016.

\bibitem{Amini2013}
A.~A. Amini, A.~Chen, P.~J. Bickel, and E.~Levina,
\newblock ``Pseudo-likelihood methods for community detection in large sparse
  networks,''
\newblock {\em Ann. Statist.}, vol. 41, no. 4, pp. 2097--2122, 2013.

\bibitem{PhysRevESSL}
P.~Zhang, C.~Moore, and L.~Zdeborov\'a,
\newblock ``Phase transitions in semisupervised clustering of sparse
  networks,''
\newblock {\em Phys. Rev. E}, vol. 90, pp. 052802, Nov 2014.

\bibitem{AvilesRivero2019}
A.I. Aviles-Rivero, N.~Papadakis, R.~Li, S.M. Alsaleh, and R.T. Tan C.-B.
  Schoenlieb,
\newblock ``Beyond supervised classification: Extreme minimal supervision with
  the graph 1-laplacian,''
\newblock {\em hal-02170176}, 2019.

\bibitem{CaiPLSBM}
T.T. Cai, T.~Liang, and A.~Rakhlin,
\newblock ``Inference via message passing on partially labeled stochastic block
  models,''
\newblock {\em arXiv}, 2016.

\bibitem{ChungSpecGraphTheory}
Fan R.\~K.\ Chung,
\newblock {\em Spectral Graph Theory},
\newblock 1997.

\bibitem{Porter2009}
M.A. Porter, J.-P. Onnela, and P.~J. Mucha,
\newblock ``Communities in networks,''
\newblock {\em Notices of the American Mathematical Society}, vol. 56, pp.
  1082--1097, 1164--1166, 2009.

\bibitem{BoydConvexBook}
S.~Boyd and L.~Vandenberghe,
\newblock {\em Convex Optimization},
\newblock Cambridge Univ. Press, Cambridge, UK, 2004.

\bibitem{JungTVMin2019}
A.~Jung, A~O. Hero, A.~Mara, S.~Jahromi, A.~Heimowitz, and Y.C. Eldar,
\newblock ``Semi-supervised learning in network-structured data via total
  variation minimization,''
\newblock {\em IEEE Trans. Signal Processing}, vol. 67, no. 24, Dec. 2019.

\bibitem{ComplexitySLP2018}
A.~Jung,
\newblock ``On the complexity of sparse label propagation,''
\newblock {\em Front. Appl. Math. Stat.}, vol. 4, pp. 22, July 2018.

\bibitem{KleinbergTardos2006}
J.~Kleinberg and E.~Tardos,
\newblock {\em Algorithm Design},
\newblock Addison Wesley, 2006.

\bibitem{JungnckelBook}
D.~Jungnickel,
\newblock {\em Graphs, Networks and Algorithms},
\newblock Springer Berlin Heidelberg, 4 edition, 2013.

\bibitem{NNSPFrontiers2018}
A.~Jung and M.~Hulsebos,
\newblock ``The network nullspace property for compressed sensing of big data
  over networks,''
\newblock {\em Front. Appl. Math. Stat.}, Apr. 2018.

\bibitem{Hoffman1960}
A.~J. Hoffman,
\newblock ``Some recent applications of the theory of linear inequalities to
  extremal combinatorial analysis,''
\newblock {\em Proc. Symp. in Applied Mathematics, Amer. Math. Soc.}, pp.
  113--127, 1960.

\bibitem{BertsekasNetworkOpt}
Dimitri~P. Bertsekas,
\newblock {\em Network Optimization: Continuous and Discrete Models},
\newblock Athena Scientific, 1998.

\bibitem{RauhutFoucartCS}
S.~Foucart and H.~Rauhut,
\newblock {\em A Mathematical Introduction to Compressive Sensing},
\newblock Springer, New York, 2012.

\bibitem{Tropp2015}
J.A. Tropp,
\newblock ``An introduction to matrix concentration inequalities,''
\newblock {\em Found. Trends Mach. Learn.}, May 2015.

\end{thebibliography}

\end{document}